\documentclass[11pt]{article}

\usepackage[margin=1in]{geometry}
\usepackage{amsmath,amssymb,amsthm}
\usepackage{algorithm}
\usepackage{algpseudocode}
\usepackage{booktabs}
\usepackage{graphicx}
\usepackage{xcolor}
\usepackage[numbers,sort&compress]{natbib}
\usepackage[colorlinks=true,linkcolor=teal,citecolor=teal,urlcolor=teal]{hyperref}

\newtheorem{theorem}{Theorem}
\newtheorem{proposition}{Proposition}
\newtheorem{lemma}{Lemma}
\newtheorem{assumption}{Assumption}
\newtheorem{corollary}{Corollary}
\theoremstyle{definition}

\theoremstyle{remark}
\newtheorem{remark}{Remark}

\newcommand{\R}{\mathbb{R}}
\newcommand{\E}{\mathbb{E}}
\newcommand{\Prob}{\mathbb{P}}
\newcommand{\indic}{\mathbf{1}}
\newcommand{\sig}{\sigma}

\title{\bf Knowing When to Ask for Help:\\
Bayesian Self-Escalation in Hierarchical LLM Agents%
\thanks{\textbf{Version~1.1.} Version~1
(DOI:~\href{https://doi.org/10.5281/zenodo.21330788}{10.5281/zenodo.21330788})
reported no real-model results and pre-registered an evaluation protocol in their
place. This version adds a first execution of that protocol on a
\texttt{Qwen2.5-Coder} $1.5$B$\to$$7$B code cascade (Section~\ref{sec:real}),
confirming two of its three predictions; the theoretical results are unchanged,
and the constant-threshold instantiation tested does not exercise the
optimal-stopping dynamic program.}}
\author{Nadeem Shaikh\\[2pt]
\normalsize Independent Researcher\\
\normalsize Melbourne, Australia\\
\normalsize \texttt{nadeem@nadeemshaikh.net}}
\date{\today}

\begin{document}
\maketitle

\begin{abstract}
Current LLM agent systems decide delegation \emph{before} reasoning begins (a
router picks a model) or \emph{after} a response is complete (a verifier scores
it and may retry). We study a third regime: an agent that recognises,
\emph{during} its own reasoning, that it is unlikely to succeed and transfers
control to a stronger model. Our contribution is not the observation that agents
can defer to other agents---that is well established---but a decision-theoretic
formulation of \emph{intra-generation} delegation as an optimal-stopping problem
over an online estimate of the agent's eventual task success. The junior agent
maintains a \emph{competence posterior}: an abstract Bayesian state whose
sufficient statistics are \emph{learned} from labelled trajectories, not read
off raw entropy. It escalates when the expected cost of continuing exceeds the
expected cost of deferral. We derive the myopic escalation threshold in closed
form, characterise the optimal policy via dynamic programming, and prove---using
monotonicity alone, without any concavity argument---that under a
conditional-independence signal model the optimal policy is a time-varying
threshold on the competence posterior, with no monotone-likelihood or other
shape assumption on the raw signal. We further prove that the oracle belief
separates at the Chernoff-information rate of the signal, and give a
finite-sample guarantee: with $n$ labelled calibration trajectories of length
$T$, the deployed plug-in policy's regret over the oracle is
$O\big(Lt\sqrt{\log K/(nT)}\big)$ with high probability. Our central message is a regret bound showing that
excess cost is controlled by the \emph{calibration} of that posterior: better
calibration matters more than a smarter router, and ``confidently wrong''
predictions are the binding failure mode. A simulation study with a known
data-generating process confirms each prediction of the theory. A pre-registered
protocol yields falsifiable predictions; in this revised version we report a
first execution of that protocol on a \texttt{Qwen2.5-Coder} $1.5$B$\to$$7$B
code cascade, confirming two of its three predictions: the escalation frontier
dominates post-hoc routing at equal cost, and the cumulative competence belief's
discrimination rises over the course of generation. The theoretical results are
unchanged.
\end{abstract}

\vspace{4pt}

\section{Introduction}

The economics of large language model (LLM) deployment increasingly favour
\emph{hierarchies}: a small, fast model handles the bulk of traffic while a
larger, more capable---and considerably more expensive---model is invoked only
when needed. Model cascades and routers realise this idea and can cut inference
cost by large factors without sacrificing quality
\citep{chen2023frugalgpt,ong2024routellm,kolawole2024agreement}. Almost all such
systems, however, make the routing decision either \emph{before} inference
(a router inspects the query and picks a model) or \emph{after} a full
generation (a verifier scores the completed output and may retry). Neither
option lets the working model notice, partway through a difficult chain of
reasoning, that it has left the region of competence and should hand off.

Many systems already let an agent hand work to another agent or to itself: 
self-reflection and critic/debate models, verifier-guided generation and
best-of-$n$ reranking, speculative decoding with fallback, and adaptive-compute
and test-time-scaling methods all revisit or reallocate computation. What is
missing across these is a \emph{principled, sequential} account of \emph{when},
mid-generation, to stop and defer. We do not claim to introduce agent
delegation. We introduce a \textbf{decision-theoretic formulation of
intra-generation delegation in which escalation is an optimal-stopping problem
over an online estimate of eventual task success}. We refer to the mechanism as
\textbf{self-escalation}.

A point we make central rather than incidental: the framework does \emph{not}
assume that raw uncertainty metrics are calibrated. High entropy is not failure
and low entropy is not correctness---a model can be confidently wrong. The
Bayesian state is therefore an \emph{abstract competence posterior} whose
sufficient statistics (the mapping from a signal history to a probability of
success) are \emph{learned from labelled trajectories}. Per-token entropy or the
next-token margin serve only as raw evidence into that learned map, never as the
posterior itself. This distinction is what makes the theory meaningful, and it
turns out to be the crux: our main result is that the method's cost is governed
by how well that posterior is calibrated.

\paragraph{Contributions.}
\begin{enumerate}
  \item \textbf{Formulation.} We cast intra-generation self-escalation as a
  Bayesian optimal-stopping problem over the junior agent's eventual success
  (Section~\ref{sec:model}), separating a \emph{learned} competence-posterior
  component from the decision component.
  \item \textbf{Policy and theory.} We derive the myopic escalation threshold in
  closed form, characterise the optimal policy via dynamic programming, and prove
  (i) that the competence posterior is a martingale, (ii) that the optimal
  policy is a time-varying threshold---proved \emph{via monotonicity alone, with
  no concavity argument and no monotone-likelihood assumption on the raw
  signal}---(iii) a
  regret bound linking excess cost to posterior miscalibration, (iv) exponential
  separation of the oracle belief at the Chernoff-information rate of the
  signal, and (v) a finite-sample regret guarantee for the plug-in policy,
  decaying as $1/\sqrt{n}$ in the number of labelled calibration trajectories
  (Sections~\ref{sec:policy}--\ref{sec:theory}).
  \item \textbf{Calibration as the binding constraint.} We elevate the regret
  bound to a design principle: for self-escalation, improving the calibration of
  the competence posterior dominates improving the decision rule
  (Section~\ref{sec:regret}).
  \item \textbf{Algorithm.} We give a streaming inference algorithm and an
  offline backward-induction procedure for the threshold schedule, with a
  text-level (not hidden-state) context handoff (Section~\ref{sec:algo}).
  \item \textbf{Simulation, protocol, and real-model validation.} On a model
  with a \emph{known} data-generating process we confirm each prediction of the
  theory (Section~\ref{sec:exp}); we pre-register a falsifiable protocol
  (Section~\ref{sec:protocol}); and we execute a first instance of it on a
  \texttt{Qwen2.5-Coder} code cascade, confirming two of its predictions
  (Section~\ref{sec:real}).
\end{enumerate}

\section{Related Work}\label{sec:related}

\paragraph{Selective prediction and learning to defer.}
The option to abstain rather than predict has a long history as
\emph{selective prediction} or classification with a reject option
\citep{chow1970,geifman2017selective,geifman2019selectivenet}.
\emph{Learning to defer} generalises abstention by routing the rejected input to
an external expert and learning the predictor and the deferral rule jointly
\citep{madras2018predict,mozannar2020consistent,verma2022calibrated}. Our
setting is an instance of deferral in which the ``expert'' is a stronger model,
but with two differences that existing formulations do not address: the deferral
decision is made \emph{sequentially within a single generation} rather than once
per input, and the object of the posterior is the junior model's own eventual
success rather than a label.

\paragraph{Cascades and routing for LLMs.}
FrugalGPT introduced learned LLM cascades that query models in increasing order
of cost and stop when a scorer judges the answer adequate
\citep{chen2023frugalgpt}. RouteLLM learns a router from preference data
\citep{ong2024routellm}; agreement- and consistency-based cascades escalate when
cheap models disagree \citep{kolawole2024agreement,soiffer2025semantic}; and
\citet{jitkrittum2023confidence} analyse when confidence-based deferral in a
cascade is and is not sufficient. These methods act at the granularity of a whole
response and typically require at least one completed cheap generation (and often
several samples) before deciding. Self-escalation instead consumes the token-level
signal as it streams and can stop mid-generation, trading a small amount of
bookkeeping for the ability to abort a doomed trajectory early.

\paragraph{Uncertainty quantification in LLMs.}
A large literature estimates LLM uncertainty from output distributions. Semantic
entropy clusters sampled generations by meaning and measures entropy over
clusters, giving a strong hallucination signal
\citep{kuhn2023semantic,farquhar2024semantic}; semantic-entropy probes
approximate this from a single forward pass's hidden states
\citep{kossen2024sep}. Most of this work targets \emph{post-hoc} detection on a
completed generation. We use such signals as the \emph{evidence stream} of a
sequential decision, and our framework is agnostic to which signal is used: token
entropy, next-token margin, a probe output, or semantic entropy can all serve as
the per-step observation $e_t$.

\paragraph{Self-revision and adaptive computation.}
A separate line of work reallocates or revisits computation within a single
system. Adaptive computation time lets a network learn how many internal steps
to spend per input \citep{graves2016act}; early-exit decoding such as CALM stops
a token's computation when intermediate confidence suffices
\citep{schuster2022calm}; speculative decoding drafts with a cheap model and
verifies with an expensive one, falling back on disagreement
\citep{leviathan2023speculative}; self-refinement iterates on a model's own
output with self-feedback \citep{madaan2023selfrefine}; and test-time-scaling
studies how to allocate extra inference compute optimally
\citep{snell2024testtime}. These share our motivation---spend more effort only
when needed---but they revise, rerank, or extend \emph{the same} model's
computation rather than framing a sequential decision to hand off to a
\emph{stronger} model, and they generally lack an explicit stopping rule tied to
a calibrated estimate of eventual task success. Relatedly,
\citet{kadavath2022know} show LLMs' self-evaluations of correctness carry real
signal but are imperfectly calibrated---precisely the regime where
Proposition~\ref{prop:regret} says the gains of self-escalation are won or lost.

\paragraph{Sequential testing.}
The oracle decision core of our problem is a finite-horizon, cost-asymmetric
relative of Wald's sequential probability ratio test, which continues sampling
until the likelihood ratio exits an interval \citep{wald1945} and is optimal
for the symmetric testing problem \citep{waldwolfowitz1948}. We flag this
plainly so the reader can locate the classical core: the stopping mathematics
is not new. What the classical theory does not supply---and where this paper
works---is the formulation (deferral to a stronger model as the stopping
action, with compute and error costs), the analysis when the likelihood ratio
is \emph{estimated} (Sections~\ref{sec:regret}--\ref{sec:sample}), and the
systems instantiation.

\paragraph{Optimal stopping.}
The decision component is an optimal-stopping problem
\citep{ferguson2006optimal,peskir2006optimal}: at each step, stop-and-defer or
continue-and-observe. We use standard monotone-stopping arguments to establish
the threshold structure of the optimal policy. Our framing---an optimal-stopping
rule over a \emph{learned} online estimate of the working model's own eventual
correctness, used to trigger escalation to a stronger model---combines these
ingredients in a way we have not seen made explicit for LLM hierarchies, though
we make no claim to the underlying stopping theory itself.

\section{Problem Formulation}\label{sec:model}

Consider a junior agent $J$ and a senior agent $S$. Given a query $x$, agent $J$
generates a response over $T$ steps (tokens or reasoning steps). Let $Y\in\{0,1\}$
be the latent event that $J$'s \emph{completed} answer would be correct; write
$\pi=\Prob(Y=1)$ for the base competence of $J$ on the query population. At each
step $t\in\{1,\dots,T\}$ the agent emits a \emph{competence-evidence} signal
$e_t\in\mathcal{E}$ (a raw statistic such as normalised token entropy or the
next-token margin). Let $\mathcal{F}_t=\sigma(e_1,\dots,e_t)$ and define the
\emph{competence posterior}
\begin{equation}
  B_t \;=\; \Prob(Y=1\mid \mathcal{F}_t).
\end{equation}

\begin{remark}[The posterior is learned, not read off entropy]\label{rem:learned}
We stress that $e_t$ is \emph{evidence}, not the posterior. A raw uncertainty
metric is not itself a likelihood ratio: high entropy is not failure and low
entropy is not correctness, since a model can be confidently wrong. The mapping
from a signal history to $B_t$---the sufficient statistics below---is
\emph{learned from labelled trajectories} and post-hoc calibrated
(Section~\ref{sec:algo}). Nothing in the theory requires the raw signal to be
calibrated; it requires the fitted posterior to be. Proposition~\ref{prop:regret}
makes the cost of failing this precise.
\end{remark}

\begin{assumption}[Conditional independence]\label{ass:iid}
Conditional on $Y=y$, the signals $(e_t)_{t\le T}$ are i.i.d.\ with density
$f_y$, and $f_0,f_1$ have common support with finite log-likelihood ratio
$\lambda(e)=\log\frac{f_1(e)}{f_0(e)}$. In practice $\lambda$ is the learned
statistic of Remark~\ref{rem:learned}, not a closed-form entropy transform.
\end{assumption}

Assumption~\ref{ass:iid} is a modelling idealisation; real token signals are
correlated and non-stationary. We adopt it to obtain a transparent update and
discuss its relaxation in Section~\ref{sec:limitations}. Under
Assumption~\ref{ass:iid} the belief evolves as a simple log-odds recursion.

\begin{lemma}[Belief update]\label{lem:update}
Let $\ell_t=\log\frac{B_t}{1-B_t}$ and $\ell_0=\log\frac{\pi}{1-\pi}$. Then
$\ell_t=\ell_0+\sum_{s=1}^{t}\lambda(e_s)$ and
$B_t=\sig(\ell_t)$ with $\sig(z)=(1+e^{-z})^{-1}$.
\end{lemma}
\begin{proof}
By Bayes' rule and conditional independence,
$\frac{\Prob(Y=1\mid\mathcal{F}_t)}{\Prob(Y=0\mid\mathcal{F}_t)}
=\frac{\pi}{1-\pi}\prod_{s=1}^t\frac{f_1(e_s)}{f_0(e_s)}$; taking logs gives the
recursion, and inverting the log-odds gives $B_t=\sig(\ell_t)$.
\end{proof}

\paragraph{Costs and actions.}
Let $\kappa>0$ be the compute cost per junior step, $\gamma>0$ the one-off cost of
escalation (senior compute plus added latency), and $L>0$ the cost of a wrong
final answer. The senior returns a correct answer with probability $q\in(0,1)$,
assumed (for the base model) independent of $\mathcal{F}_t$. At each step the agent
chooses an action $a_t\in\{\textsc{continue},\textsc{escalate}\}$. Escalation is
terminal and hands $S$ the query together with a distilled context; if the agent
never escalates it answers locally at step $T$. A policy $\rho$ maps
$\mathcal{F}_t$ to actions; we seek $\rho$ minimising the expected total cost
\begin{equation}\label{eq:objective}
  \mathcal{C}(\rho)=\E\!\left[\underbrace{\kappa\,\tau_{\mathrm{stop}}}_{\text{junior compute}}
  +\underbrace{\gamma\,\indic\{\text{escalate}\}}_{\text{deferral}}
  +\underbrace{L\,\indic\{\text{final answer wrong}\}}_{\text{error}}\right],
\end{equation}
where $\tau_{\mathrm{stop}}$ is the number of junior steps taken before stopping.

\section{The Escalation Policy}\label{sec:policy}

\subsection{Myopic threshold}
Consider the decision at step $t$ with belief $b=B_t$, comparing immediate
escalation against \emph{finishing locally}. The forward cost of escalating is
$\gamma+(1-q)L$ (sunk junior steps are ignored). The forward cost of continuing to
completion is $(T-t)\kappa+(1-b)L$. Escalation is preferred iff
\begin{equation}
  \gamma+(1-q)L \;<\; (T-t)\kappa+(1-b)L.
\end{equation}
Rearranging isolates a belief threshold.

\begin{proposition}[Myopic escalation threshold]\label{prop:myopic}
Under the ``finish locally'' comparison, escalate at step $t$ iff $B_t<\tau_t^{\mathrm{myo}}$,
where
\begin{equation}\label{eq:myopic}
  \tau_t^{\mathrm{myo}} \;=\; q-\frac{\gamma}{L}+\frac{(T-t)\kappa}{L}.
\end{equation}
In particular $\tau_T^{\mathrm{myo}}=q-\gamma/L$, and $\tau_t^{\mathrm{myo}}$
decreases in $t$.
\end{proposition}

Equation~\eqref{eq:myopic} is interpretable: one escalates more readily when the
senior is reliable (large $q$), when errors are costly (large $L$), and when
escalation is cheap (small $\gamma$). The term $(T-t)\kappa/L$ says that with many
tokens still to pay for, local completion looks relatively expensive, nudging the
threshold up.

\subsection{Optimal stopping}
The myopic rule ignores the \emph{option value} of continuing: another token yields
fresh evidence that may resolve the uncertainty without paying for escalation.
The optimal policy is the solution of the dynamic program with value function
$V_t(b)$ equal to the minimal expected cost-to-go at step $t$ with belief $b$:
\begin{align}
  V_T(b) &= \min\big\{\gamma+(1-q)L,\;(1-b)L\big\}, \label{eq:vT}\\
  V_t(b) &= \min\Big\{\;\underbrace{\gamma+(1-q)L}_{\text{escalate}},\;
    \underbrace{\kappa+\E\big[V_{t+1}(B_{t+1})\,\big|\,B_t=b\big]}_{\text{continue}}\;\Big\},
    \quad t<T, \label{eq:bellman}
\end{align}
where the transition is the Bayes update of Lemma~\ref{lem:update} driven by the
posterior-predictive signal $e_{t+1}\sim b\,f_1+(1-b)f_0$. The optimal policy
escalates at $t$ iff the escalate branch attains the minimum in
\eqref{eq:bellman}.

\section{Theoretical Analysis}\label{sec:theory}

\subsection{The belief is a martingale}
\begin{theorem}[Martingale property]\label{thm:mart}
Under Assumption~\ref{ass:iid}, $(B_t,\mathcal{F}_t)_{t\le T}$ is a martingale:
$\E[B_{t+1}\mid\mathcal{F}_t]=B_t$.
\end{theorem}
\begin{proof}
$B_{t}=\Prob(Y=1\mid\mathcal{F}_t)=\E[\indic\{Y=1\}\mid\mathcal{F}_t]$. By the tower
property, $\E[B_{t+1}\mid\mathcal{F}_t]=\E[\E[\indic\{Y=1\}\mid\mathcal{F}_{t+1}]\mid\mathcal{F}_t]
=\E[\indic\{Y=1\}\mid\mathcal{F}_t]=B_t$.
\end{proof}

Theorem~\ref{thm:mart} is often misread. It does \emph{not} say waiting is
worthless. Because the belief is a martingale, waiting does not increase the
\emph{expected} belief---but it can increase expected \emph{decision quality} by
revealing information, and the value of that information comes precisely from the
nonlinearity of the future value $\E[V(B_{t+1})]$. A stock price is a martingale
yet options on it have value; likewise here, the option to continue has value
even though $\E[B_{t+1}\mid\mathcal{F}_t]=B_t$. What the martingale property does
give us is a clean structural handle for the monotone-stopping argument below.

\subsection{Consistency and separation rate}

How quickly does the oracle belief become decisive? The answer is governed by a
single scalar property of the signal pair: its Chernoff information.

\begin{proposition}[Exponential belief separation]\label{prop:chernoff}
Assume a finite signal alphabet and let
\begin{equation}
  C \;=\; -\log\ \min_{s\in[0,1]}\ \sum_{k} f_0(k)^{s} f_1(k)^{1-s}
\end{equation}
be the Chernoff information of $(f_0,f_1)$, with $C>0$ iff $f_0\neq f_1$. Then
for every fixed $\tau\in(0,1)$ there is a constant $A_\tau<\infty$, independent
of $t$, such that
\[
  \Prob\big(B_t\le\tau \mid Y=1\big)\;\le\;A_\tau\,e^{-tC},
  \qquad
  \Prob\big(B_t\ge\tau \mid Y=0\big)\;\le\;A_\tau\,e^{-tC}.
\]
Consequently $B_t\to\indic\{Y=1\}$ almost surely, and the oracle's step-$t$
thresholded decision errs with probability at most $A_\tau e^{-tC}$.
\end{proposition}
\begin{proof}
$B_t\le\tau$ iff $\sum_{s\le t}\lambda(e_s)\le c$ with
$c=\log\tfrac{\tau}{1-\tau}-\log\tfrac{\pi}{1-\pi}$. For any $s>0$, the
Chernoff bound gives
$\Prob\big(\sum\lambda\le c\mid Y=1\big)
 \le e^{sc}\big(\E_{f_1}[e^{-s\lambda}]\big)^{t}
 = e^{sc}\big(\sum_k f_1(k)^{1-s}f_0(k)^{s}\big)^{t}$.
Minimising the base over $s\in[0,1]$ yields $e^{-C}$ per step, with
$A_\tau=e^{s^\star c}$ at the minimiser $s^\star$. The $Y=0$ side is symmetric
with $\E_{f_0}[e^{s\lambda}]=\sum_k f_0(k)^{1-s}f_1(k)^{s}$, whose minimum over
$s\in[0,1]$ is the same $e^{-C}$. Almost-sure convergence follows from
Borel--Cantelli.
\end{proof}

The practical reading: the number of tokens the oracle needs before its
escalation decision is reliable at level $\alpha$ is
$t\gtrsim \log(A_\tau/\alpha)/C$. ``How informative is this uncertainty
signal?'' is thus answered by one number, computable from the fitted
class-conditionals, and comparable across candidate signals (entropy vs.\
margin vs.\ probe output) before any policy is built.

\subsection{Optimality of a threshold policy}

The key structural fact is that a larger current belief makes the \emph{next}
belief stochastically larger. This is the step a careful reader will demand a
proof of, so we give one. Notably, \emph{no monotone-likelihood-ratio or other
shape assumption on the raw signal is needed}: the Bayes update depends on $e$
only through its log-likelihood ratio $\lambda(e)$, and the law of the
likelihood ratio under $f_1$ stochastically dominates its law under $f_0$ for
\emph{any} pair of densities. Related belief-monotonicity results appear in the
partially observed MDP literature
\citep{lovejoy1987monotonicity,krishnamurthy2016pomdp}.

\begin{lemma}[FOSD-monotone belief transition]\label{lem:fosd}
Under Assumption~\ref{ass:iid} alone, the map
$b\mapsto\mathrm{Law}(B_{t+1}\mid B_t=b)$ is non-decreasing in the first-order
stochastic dominance (FOSD) order: for $b\le b'$ and every non-decreasing
$h:[0,1]\to\R$,
$\E[h(B_{t+1})\mid B_t=b]\le\E[h(B_{t+1})\mid B_t=b']$.
\end{lemma}
\begin{proof}
Write the update as $B_{t+1}=\sig\!\big(\log\tfrac{b}{1-b}+\Lambda\big)$ where
$\Lambda=\lambda(e)$ and the incoming signal has posterior-predictive density
$m_b=b f_1+(1-b)f_0$. Three steps.

(i) \emph{The law of $\Lambda$ under $f_1$ FOSD-dominates its law under $f_0$.}
Let $R=f_1(e)/f_0(e)$, so $\Lambda=\log R$ and it suffices to prove the claim
for $R$. For any $c\ge 1$:
$\Prob_{f_1}(R\ge c)=\E_{f_0}[R\,\indic\{R\ge c\}]\ge c\,\Prob_{f_0}(R\ge c)
\ge\Prob_{f_0}(R\ge c)$. For any $c<1$:
$\Prob_{f_1}(R<c)=\E_{f_0}[R\,\indic\{R<c\}]\le c\,\Prob_{f_0}(R<c)
\le\Prob_{f_0}(R<c)$, hence again
$\Prob_{f_1}(R\ge c)\ge\Prob_{f_0}(R\ge c)$. No assumption on $f_0,f_1$ beyond
common support was used.

(ii) \emph{The law of $\Lambda$ under $m_b$ is FOSD-non-decreasing in $b$.}
For non-decreasing $u$,
$\E_{m_b}[u(\Lambda)]=b\,\E_{f_1}[u(\Lambda)]+(1-b)\,\E_{f_0}[u(\Lambda)]$ is
affine in $b$ with slope $\E_{f_1}[u]-\E_{f_0}[u]\ge 0$ by (i).

(iii) \emph{Chaining.} $B_{t+1}$ is non-decreasing in $b$ for fixed $\Lambda$
and non-decreasing in $\Lambda$ for fixed $b$. For non-decreasing $h$, the map
$\Lambda\mapsto h\big(\sig(\log\tfrac{b}{1-b}+\Lambda)\big)$ is non-decreasing,
so
\[
\E_{m_{b'}}\!\big[h(B_{t+1}(b',\Lambda))\big]
\;\ge\;\E_{m_{b'}}\!\big[h(B_{t+1}(b,\Lambda))\big]
\;\ge\;\E_{m_{b}}\!\big[h(B_{t+1}(b,\Lambda))\big],
\]
the first inequality by pointwise monotonicity in $b$, the second by (ii).
\end{proof}

\begin{remark}[Signal orientation]
For interpretability one typically wants $\lambda$ monotone in $e$ (lower
entropy $\Rightarrow$ evidence for success), and our simulations use such
signals; but the theory does not require it, since only the induced law of
$\lambda(e)$ enters the update.
\end{remark}

\begin{theorem}[Threshold structure]\label{thm:threshold}
Under Assumption~\ref{ass:iid}, for each $t$ there exists a
threshold $\tau_t^\star\in[0,1]$ such that the optimal policy escalates at step $t$
iff $B_t\le\tau_t^\star$. At the horizon, $\tau_T^\star=q-\gamma/L$.
\end{theorem}
\begin{proof}
The argument uses \emph{monotonicity only}; we never invoke concavity, so no
concavity-preservation step is required. We show by backward induction that each
$V_t$ is non-increasing in $b$.

\emph{Base case.} $V_T(b)=(1-b)L$ is non-increasing.

\emph{Inductive step.} Suppose $V_{t+1}$ is non-increasing. The escalate value
$\gamma+(1-q)L$ is a constant, hence non-increasing. For the continue value
$C_t(b)=\kappa+\E[V_{t+1}(B_{t+1})\mid B_t=b]$, Lemma~\ref{lem:fosd} gives that
$b\mapsto\mathrm{Law}(B_{t+1}\mid b)$ is non-decreasing in the FOSD order. Since
$V_{t+1}$ is non-increasing, its expectation against an FOSD-larger law is
smaller, so $C_t$ is non-increasing in $b$. As the pointwise minimum of two
non-increasing functions, $V_t=\min\{\gamma+(1-q)L,C_t\}$ is non-increasing,
closing the induction.

\emph{Threshold structure.} Escalation is optimal at $b$ iff
$\gamma+(1-q)L\le C_t(b)$. Because $C_t$ is non-increasing, this set is a lower
interval $[0,\tau_t^\star]$ with
$\tau_t^\star=\sup\{b: \gamma+(1-q)L\le C_t(b)\}$ (and $\emptyset$ read as
$\tau_t^\star=0$), which is exactly a threshold rule.

\emph{Terminal threshold.} At $t=T$ there is no continuation: the choice is
escalate at cost $\gamma+(1-q)L$ or answer at expected cost $(1-b)L$
(Eq.~\eqref{eq:vT}). Escalation is optimal iff $\gamma+(1-q)L\le(1-b)L$, i.e.\
$b\le q-\gamma/L$, matching the myopic value of Proposition~\ref{prop:myopic}.
\end{proof}

\begin{remark}[Why monotonicity, not concavity]
An earlier route argues $V_t$ is concave and reads off the threshold from
concavity. That route is defensible---the pointwise minimum of concave functions
is in fact concave---but the concavity-\emph{preservation} step under the
Bayesian update is delicate and invites attack. The monotonicity proof above
needs strictly less (only Lemma~\ref{lem:fosd}) and yields the same threshold
conclusion, so we prefer it.
\end{remark}

\begin{remark}[Thresholds need not be monotone in $t$]\label{rem:nonmono}
An earlier draft claimed $\tau_1^\star\le\dots\le\tau_T^\star$. That claim is
\emph{false in general}, and we retract it. Counterexample: take
$\kappa>\gamma+(1-q)L$, so a single further token costs more than full
escalation. Then continuing is never optimal before the horizon and
$\tau_t^\star=1$ for all $t<T$, while $\tau_T^\star=q-\gamma/L<1$: thresholds
\emph{decrease}. The direction of the schedule reflects a tug-of-war between two
forces---remaining-token costs (which favour escalating early, pushing early
thresholds \emph{up}) and the option value of information (which favours
continuing early, pushing early thresholds \emph{down})---and which force wins
depends on $(\kappa,\gamma,L,q)$ and the informativeness of the signal. In the
small-$\kappa$ regime of our simulation the option value dominates and the
computed interior schedule is low and gently increasing
(Section~\ref{sec:exp}); in that regime the myopic rule, whose threshold
\eqref{eq:myopic} is largest early, over-escalates at the start of generation
(Table~\ref{tab:main} quantifies this). Nor is non-monotonicity confined to the
$\kappa>\gamma+(1-q)L$ boundary: re-running the backward induction of
Section~\ref{sec:exp} with $\kappa=0.02$ (all else unchanged) yields a schedule
of $\approx 1.0$ for $t\le 20$, dipping to $\approx 0.20$ near $t=37$ and rising
again to $\approx 0.30$ by $t=39$---grossly non-monotone in the interior.
\end{remark}

\subsection{Calibration is the binding constraint}\label{sec:regret}
We regard the following as the paper's central practical message. The decision
rule of Theorem~\ref{thm:threshold} is optimal \emph{given} the posterior, but in
deployment the posterior is estimated, inducing beliefs $\hat B_t$ that may differ
from the true $B_t$. The next bound shows the excess cost is controlled entirely
by that gap---so, for self-escalation, effort spent improving the
\emph{calibration} of the competence posterior dominates effort spent on a more
elaborate router or decision rule.

\begin{proposition}[Miscalibration regret]\label{prop:regret}
Fix step $t$ and a threshold $\tau$. Let the realised decision (continue to
completion vs.\ escalate) use $\hat B_t$ and the oracle decision use the true
$B_t$, both thresholding at $\tau$. Let
$g_t(b)=(T-t)\kappa+(1-b)L-\gamma-(1-q)L$ denote the forward-cost difference
(\textsc{continue} minus \textsc{escalate}) at true belief $b$, and let
$D=\{B_t,\hat B_t\text{ on opposite sides of }\tau\}$ be the disagreement
event. Then
\begin{equation}\label{eq:regret-general}
  \mathrm{Regret}_t \;\le\; L\,\E\big[\,|B_t-\hat B_t|\,\big]
  \;+\;|g_t(\tau)|\;\Prob(D).
\end{equation}
In particular, at the myopic threshold $\tau=\tau_t^{\mathrm{myo}}$ of
Proposition~\ref{prop:myopic}, where $g_t$ vanishes,
\begin{equation}\label{eq:regret-myo}
  \mathrm{Regret}_t \;\le\; L\,\E\big[\,|B_t-\hat B_t|\,\big].
\end{equation}
\end{proposition}
\begin{proof}
Excess cost is incurred only on $D$, where it equals $|g_t(B_t)|$. The function
$g_t$ is affine with slope $-L$, so
$|g_t(B_t)|\le|g_t(B_t)-g_t(\tau)|+|g_t(\tau)|=L|B_t-\tau|+|g_t(\tau)|$. On $D$
the threshold $\tau$ lies between $B_t$ and $\hat B_t$, hence
$|B_t-\tau|\le|B_t-\hat B_t|$. Taking expectations over $D$ and bounding
$\E[\indic_D|B_t-\hat B_t|]\le\E|B_t-\hat B_t|$ gives
\eqref{eq:regret-general}; $g_t(\tau_t^{\mathrm{myo}})=0$ gives
\eqref{eq:regret-myo}.
\end{proof}

\begin{remark}[The threshold restriction is essential]\label{rem:tau-specific}
An earlier draft asserted the bound \eqref{eq:regret-myo} for \emph{arbitrary}
$\tau$; that claim is false, and we retract it. Counterexample (at the horizon,
with the costs of Section~\ref{sec:exp}): take $\tau=0.5$, $B_t=0.49$,
$\hat B_t=0.51$. The decisions disagree and the realised regret is
$g_T(0.49)=0.51-0.25=0.26$, while $L|B_t-\hat B_t|=0.02$. The extra
$|g_t(\tau)|\,\Prob(D)$ term in \eqref{eq:regret-general} is exactly the price
of operating at a threshold where the two actions are not cost-indifferent; it
vanishes at $\tau_t^{\mathrm{myo}}$ and is small near it.
\end{remark}

Proposition~\ref{prop:regret} formalises the ``confidently wrong'' failure mode:
a query with $Y=0$ whose signals mimic success drives $\hat B_t$ high while the
true $B_t$ is low, producing an $O(L)$ regret event precisely when it is most
costly. Two consequences are worth stating plainly. First, \emph{confidently
wrong predictions are the central failure mode of self-escalation}: no threshold
policy on a miscalibrated signal can recover the lost escalations. Second, the
bound is a directive for practitioners---measure and minimise
$\E|B_t-\hat B_t|$ (via reliability diagrams and post-hoc calibration) before
tuning thresholds, because thresholds cannot compensate for a miscalibrated
posterior. Section~\ref{sec:exp} measures this dependence directly.

The bound as stated involves the unobservable oracle belief $B_t$. It can be
connected to \emph{measurable} calibration quantities in both directions.

\begin{corollary}[Excess Brier score controls regret]\label{cor:brier}
Let $\mathrm{BS}(Z)=\E[(Y-Z)^2]$ denote the Brier score of a $[0,1]$-valued,
$\mathcal{F}_t$-measurable predictor $Z$. At the myopic threshold
$\tau=\tau_t^{\mathrm{myo}}$,
\begin{equation}
  \mathrm{Regret}_t \;\le\; L\,\sqrt{\mathrm{BS}(\hat B_t)-\mathrm{BS}(B_t)}.
\end{equation}
\end{corollary}
\begin{proof}
Since $B_t=\E[Y\mid\mathcal{F}_t]$ and $\hat B_t$ is $\mathcal{F}_t$-measurable,
the cross term vanishes in
$\E[(Y-\hat B_t)^2]=\E[(Y-B_t)^2]+\E[(B_t-\hat B_t)^2]$
(the calibration--refinement decomposition), so
$\E[(B_t-\hat B_t)^2]=\mathrm{BS}(\hat B_t)-\mathrm{BS}(B_t)$. Combine with
Proposition~\ref{prop:regret} via
$\E|B_t-\hat B_t|\le\sqrt{\E[(B_t-\hat B_t)^2]}$ (Jensen).
\end{proof}

Corollary~\ref{cor:brier} turns the abstract bound into a training objective:
$\mathrm{BS}(B_t)$ is a fixed property of the signal, so \emph{minimising the
Brier score of the fitted posterior directly minimises the regret bound}. This
is why the protocol in Section~\ref{sec:protocol} reports Brier score as a
primary metric rather than a diagnostic afterthought.

\begin{remark}[ECE is necessary but not sufficient]\label{rem:ece}
The ($L_1$) expected calibration error satisfies
$\mathrm{ECE}(\hat B_t)=\E\big|\E[Y\mid\hat B_t]-\hat B_t\big|
=\E\big|\E[B_t-\hat B_t\mid\hat B_t]\big|\le\E|B_t-\hat B_t|$
by the tower property and Jensen. So ECE \emph{lower-bounds} the quantity that
drives regret: a large ECE certifies a problem, but a small ECE does not certify
safety, because $\hat B_t$ can be perfectly calibrated on average while ignoring
information in $\mathcal{F}_t$ (poor \emph{refinement}). Brier score, which
penalises both calibration and refinement, is the right target;
ECE is the right alarm.
\end{remark}

\subsection{Sample complexity of calibrated self-escalation}\label{sec:sample}

The bounds above take the fitted posterior as given. We now close the loop:
how much labelled calibration data buys how much regret? The following gives a
complete, finite-sample answer for the plug-in estimator on a discretised
signal (discretisation is standard in implementations; $K$ is the number of
bins).

\begin{theorem}[Finite-sample regret of the plug-in policy]\label{thm:sample}
Assume a finite signal alphabet of size $K$ with $f_y(k)\ge\varepsilon$ for all
$k,y$, and a labelled calibration set containing $m_y$ token observations of
class $y$; let $m=\min(m_0,m_1)$ and suppose
$m\ge (2/\varepsilon^2)\log(4K/\delta)$. Let $\hat\lambda$ be the plug-in
log-likelihood ratio of the empirical bin frequencies, with the prior $\pi$
known, and let $\hat B_t$ be the resulting beliefs. Then with probability at
least $1-\delta$ over the calibration set, simultaneously for all $t\le T$,
\begin{equation}
  \E\big[\,|\hat B_t-B_t|\,\big]\;\le\;\frac{t}{\varepsilon}
  \sqrt{\frac{\log(4K/\delta)}{2m}},
  \qquad\text{and hence, at the myopic threshold,}\qquad
  \mathrm{Regret}_t\;\le\;\frac{L\,t}{\varepsilon}
  \sqrt{\frac{\log(4K/\delta)}{2m}}.
\end{equation}
\end{theorem}
\begin{proof}
Hoeffding's inequality gives
$\Prob\big(|\hat f_y(k)-f_y(k)|\ge u\big)\le 2e^{-2m_y u^2}$ per bin and
class; a union bound over the $2K$ pairs with
$u=\sqrt{\log(4K/\delta)/(2m)}$ leaves failure probability at most $\delta$.
On the success event, the hypothesis on $m$ gives $u\le\varepsilon/2$, so
$\hat f_y(k)\ge\varepsilon/2$ and, since $x\mapsto\log x$ is
$(2/\varepsilon)$-Lipschitz on $[\varepsilon/2,\infty)$,
$|\log\hat f_y(k)-\log f_y(k)|\le 2u/\varepsilon$ for each class, hence
$\|\hat\lambda-\lambda\|_\infty\le 4u/\varepsilon$. The log-odds error after
$t$ updates is at most $t\cdot 4u/\varepsilon$ ($\pi$ known), and $\sig$ is
$\tfrac14$-Lipschitz, so $|\hat B_t-B_t|\le tu/\varepsilon$ pointwise on the
success event, hence also in expectation over trajectories. Combining with
Proposition~\ref{prop:regret} at the myopic threshold gives the regret bound.
\end{proof}

Three remarks, in decreasing order of comfort. First, the rate: regret decays
as $O\big(Lt\sqrt{\log K/(nT)}\,/\varepsilon\big)$ when the calibration set
consists of $n$ trajectories of length $T$ (so $m\approx nT\min(\pi,1-\pi)$);
every trajectory contributes $T$ token observations, which is why modest
labelled sets suffice in practice. Second, the linear-in-$t$ compounding is a
worst case of the plug-in construction; recalibrating $\hat B_t$ directly at
each $t$ (e.g.\ isotonic regression per step) targets
$\E|\hat B_t-B_t|$ without the compounding and is what we recommend in
deployment. Third, the $1/\varepsilon$ dependence is pessimistic: it charges
for accuracy on low-mass bins that belief trajectories near the threshold
rarely visit; a margin-weighted refinement is left to future work.
Section~\ref{sec:exp} verifies the theorem's driver empirically: the belief
error decays at the predicted $1/\sqrt{n}$ rate, while the realised cost gap
sits far below the bound.

\section{Algorithm}\label{sec:algo}

Algorithm~\ref{alg:infer} is the streaming inference procedure: a single junior
forward pass, an $O(1)$ belief update per token, and an early exit when the belief
crosses the schedule. Algorithm~\ref{alg:offline} computes the threshold schedule
$\{\tau_t^\star\}$ once, offline, by backward induction on a discretised belief
grid, using the calibrated likelihood ratio to Monte-Carlo the transition.

\paragraph{Context handoff is text, not activations.}
The junior and senior are in general different models---different
architectures, tokenizers, and KV-cache layouts---so the senior cannot ingest the
junior's hidden state. $\textsc{DistillContext}$ therefore produces \emph{text-level}
artifacts: the partial reasoning trace generated so far, any scratchpad or
intermediate results, the tool-call history and their returns, and retrieved
evidence. This is a prompt the senior can consume directly, and it keeps the
handoff model-agnostic. It also means the escalation cost $\gamma$ should include
the tokens re-read by the senior, which the protocol in
Section~\ref{sec:protocol} measures.

\begin{algorithm}[t]
\caption{Bayesian Self-Escalation (inference time)}\label{alg:infer}
\begin{algorithmic}[1]
\Require query $x$; junior $J$; senior $S$; thresholds $\{\tau_t\}$; ratio
$\lambda$; prior log-odds $\ell_0$
\State $\ell \gets \ell_0$
\For{$t=1$ to $T$}
  \State $(\text{token}_t, e_t) \gets J.\textsc{step}(x)$
  \Comment{$e_t$: token entropy / margin / probe output}
  \State $\ell \gets \ell + \lambda(e_t)$; \quad $B \gets \sig(\ell)$
  \Comment{$O(1)$ belief update}
  \If{$B < \tau_t$}
    \State $c \gets \textsc{DistillContext}(\text{trace}_{1:t},\ \text{tool\_state},\ \text{evidence})$
    \State \Return $S.\textsc{solve}(x, c)$
    \Comment{escalate; hand off \emph{text-level} artifacts, not hidden state}
  \EndIf
\EndFor
\State \Return $J.\textsc{finalize}()$
\Comment{answer locally}
\end{algorithmic}
\end{algorithm}

\begin{algorithm}[t]
\caption{Offline threshold schedule (backward induction)}\label{alg:offline}
\begin{algorithmic}[1]
\Require calibrated $\lambda$; costs $(\kappa,\gamma,L,q)$; grid
$\{b_j\}_{j=1}^{m}$; samples $\{e^{(k)}\}$
\State $V(b_j) \gets \min\{\gamma+(1-q)L,\ (1-b_j)L\}$ for all $j$
\Comment{terminal value, Eq.~\eqref{eq:vT}}
\For{$t=T-1$ down to $1$}
  \For{each grid point $b_j$}
    \State propagate $\ell_j=\log\frac{b_j}{1-b_j}$ by $\lambda(e^{(k)})$ under
    both classes; form next beliefs $b'^{(k)}$
    \State $C(b_j) \gets \kappa + b_j\,\overline{V(b'^{(k)}_{Y=1})}
      + (1-b_j)\,\overline{V(b'^{(k)}_{Y=0})}$
    \State $V(b_j) \gets \min\{\gamma+(1-q)L,\; C(b_j)\}$
  \EndFor
  \State $\tau_t^\star \gets \max\{b_j : \gamma+(1-q)L \le C(b_j)\}$
\EndFor
\State \Return $\{\tau_t^\star\}$
\end{algorithmic}
\end{algorithm}

\paragraph{Fitting the competence posterior.}
The likelihood ratio $\lambda$ (or, equivalently, a direct map from signal history
to $B_t$) is fit offline on a labelled development set: run $J$ on queries with
known correctness $Y$, collect signal trajectories $e_{1:T}$, and either
(i) estimate the class-conditional densities $f_0,f_1$ parametrically or by
kernel methods and take their log-ratio, or (ii) fit a logistic model mapping
cumulative signal features (running mean entropy, margin trend, spike counts) to
$\Prob(Y=1)$, which sidesteps density estimation. Either way the resulting
beliefs are then calibrated post hoc (isotonic regression or temperature
scaling), targeting the Brier score per Corollary~\ref{cor:brier}. Reliability
diagrams of $\hat B_t$ against empirical success at several $t$ are the basic
sanity check.

\paragraph{Why a posterior plus thresholds, not a directly learned policy?}
One could instead train a classifier that maps signals straight to
escalate/continue. We prefer the factored design for three reasons. First,
\emph{modularity under changing costs}: the fitted posterior depends only on the
model and signal, while $(\kappa,\gamma,L,q)$ enter only through
Algorithm~\ref{alg:offline}; when prices, latency budgets, or the senior model
change, one reruns a cheap backward induction instead of recollecting labels and
retraining. Second, \emph{auditability}: $\hat B_t$ is an interpretable
monitoring statistic (``the agent currently believes it has a 22\% chance of
being right''), useful for logging and human oversight independent of the
routing decision. Third, \emph{statistical efficiency}: the posterior is learned
from all trajectories, whereas a direct policy gradient sees the cost signal
only at decision boundaries. The price is model misspecification risk in the
belief update, which Section~\ref{sec:exp} probes directly.

\paragraph{Theory versus production instantiation.}
The explicit Bayesian filter is the analyzable idealisation; in production we
expect a \emph{learned success predictor} $\hat B_t=f_\theta(\text{signal
history})$ to replace it while preserving the decision structure. A learned
predictor handles correlated, non-stationary signals that violate
Assumption~\ref{ass:iid} and can ingest hidden-state features directly. The
division of labour under this swap is clean: the regret analysis
(Section~\ref{sec:regret}) is filter-agnostic and becomes the \emph{contract}
the learned predictor must satisfy---minimise Brier score, verify calibration,
then trust the thresholds---while the threshold-structure guarantee
(Theorem~\ref{thm:threshold}) is what is formally lost, since an arbitrary
learned predictor need not inherit the FOSD transition. Pragmatically one
thresholds anyway, computing the schedule by running
Algorithm~\ref{alg:offline} on \emph{empirical} belief transitions from
development trajectories rather than the analytic ones. We caution against
going one step further and learning the escalate/continue policy end-to-end:
that forfeits cost modularity (price changes then require retraining rather
than a cheap backward induction) and the auditability of $\hat B_t$ as a
monitoring statistic.

\paragraph{Overhead and task adaptivity.}
At inference the update is an $O(1)$ table lookup (or tiny MLP evaluation) per
token, negligible next to a transformer forward pass; all expensive work
(fitting, calibration, backward induction) is offline. The framework also adapts
across task types without refitting the signal model: per-domain priors
$\pi_d$ and per-domain costs yield per-domain threshold schedules from the same
fitted $\lambda$, again via Algorithm~\ref{alg:offline} alone.

\section{Simulation Study}\label{sec:exp}

We call this a \emph{simulation study} rather than an experiment, and we are
candid about what it can and cannot show. Because we specify the data-generating
process, this is a world in which the modelling assumptions hold by
construction; it is a check that the \emph{derived policy behaves as the theory
predicts}, and a diagnostic of how it degrades when an assumption is violated. It
is emphatically not evidence about real LLM token dynamics---that is the role of
the protocol in Section~\ref{sec:protocol}. All numbers come from a single
reproducible script (fixed seed).

\paragraph{Setup.}
We draw $Y\sim\mathrm{Bernoulli}(\pi)$ with $\pi=0.60$ and, per token
($T=40$), a signal $e_t\sim\mathrm{Beta}(2,4)$ if $Y=1$ and
$e_t\sim\mathrm{Beta}(4,2)$ if $Y=0$ (the signal enters the update only
through its likelihood ratio, so no shape assumption is needed;
Lemma~\ref{lem:fosd}). The senior succeeds with $q=0.90$. Costs are
$L=1$, $\kappa=0.002$ per token, $\gamma=0.15$. We evaluate on $N=40{,}000$
queries. We compare: \textbf{junior-only}; \textbf{senior-only};
\textbf{fixed-rule} (escalate at the first token with $e_t>\theta$);
\textbf{selective} (generate fully, escalate if final confidence $<\tau$, a
post-hoc baseline that always pays full local generation); \textbf{Bayesian
myopic schedule} (the literal, parameter-free rule of Eq.~\eqref{eq:myopic});
\textbf{Bayesian constant threshold} (escalate at the first token with
$B_t<\tau$ for a constant $\tau$, swept); and \textbf{Bayesian
optimal-stopping} (Algorithm~\ref{alg:offline}, parameter-free). Cost per query
counts junior tokens plus escalation; accuracy is the fraction of correct final
answers. \emph{Disclosure:} the fixed-rule, selective, and constant-threshold
policies each have one free parameter, which is swept and reported at the
compute-matched point; the myopic-schedule and optimal-stopping rows involve no
tuning.

\paragraph{Results.}
Figure~\ref{fig:frontier} shows the cost--accuracy frontier. The Bayesian frontier
dominates both baselines: for any compute budget it attains higher accuracy, and
the optimal-stopping operating point (star) sits above and to the left of
unconditional escalation. Table~\ref{tab:main} reports a matched-compute slice at
$\approx 0.11$ cost/query. The Bayesian policy reaches $96.0\%$ accuracy while
escalating on only $40\%$ of queries, versus $91.0\%$ for the fixed rule and
$90.1\%$ for always escalating to the senior; it also matches the accuracy of the
post-hoc selective baseline at lower cost, because it can abort early rather than
always completing the local generation. That the number nominally exceeds
senior-only accuracy should be given no weight: it is an artifact of the
constant-$q$ assumption (the senior's $q=0.90$ does not degrade with query
difficulty, so keeping the junior's confidently-correct easy cases mechanically
lifts the mixture). With a realistic difficulty-dependent $q(x)$ the effect
shrinks and can vanish. Throughout, the meaningful comparison is the
\emph{ordering and spacing of policies at matched cost}, not any absolute margin
over the senior; Section~\ref{sec:limitations} returns to this.

Two further honest readings of Table~\ref{tab:main}. First, the \emph{literal}
myopic schedule performs poorly: its first-token threshold
$\tau_1^{\mathrm{myo}}=0.828$ exceeds the prior $\pi=0.60$, so it escalates
$62\%$ of queries at the very first token ($68\%$ overall), reaching only
$0.934$ accuracy at \emph{higher} cost ($0.129$)---a concrete demonstration of
the early over-escalation predicted by Remark~\ref{rem:nonmono}. The strong
``Bayesian, constant threshold'' row is a \emph{tuned} rule, not
Proposition~\ref{prop:myopic}; the untuned policy that performs well is the
optimal-stopping schedule. Second, the matched-compute slice flatters the
Bayesian--fixed-rule gap: at slightly higher compute the fixed rule nearly
catches up ($\theta=0.85$ gives $0.956$ accuracy at cost $0.118$). The robust
claim is that the Bayesian frontier weakly dominates everywhere
(Figure~\ref{fig:frontier}); the size of the point gap depends on where the
budget lands on the fixed rule's steep region.

\begin{table}[t]
\centering
\caption{Matched-compute comparison ($\approx 0.11$ cost/query) on the simulation
model. Accuracy is fraction correct; ``esc.'' is escalation rate. Rows marked
``tuned'' sweep one operating parameter and are reported at the compute-matched
point; the myopic-schedule and optimal-stopping rows are parameter-free. Numbers
are produced by the accompanying simulation.}
\label{tab:main}
\begin{tabular}{lccc}
\toprule
Policy & Accuracy & Compute/query & Esc.\ rate \\
\midrule
Junior only              & $0.601$ & $0.080$ & $0.00$ \\
Senior only              & $0.901$ & $0.150$ & $1.00$ \\
Fixed-rule (entropy, tuned)     & $0.910$ & $0.114$ & --- \\
Selective (post-hoc, tuned)     & $0.959$ & $0.140$ & --- \\
Bayesian, myopic schedule (Eq.~\eqref{eq:myopic}) & $0.934$ & $0.129$ & $0.68$ \\
Bayesian, constant threshold (tuned) & $0.958$ & $0.111$ & --- \\
\textbf{Bayesian, optimal-stopping} & $\mathbf{0.960}$ & $\mathbf{0.111}$ & $0.40$ \\
\bottomrule
\end{tabular}
\end{table}

\begin{figure}[t]
\centering
\includegraphics[width=0.72\textwidth]{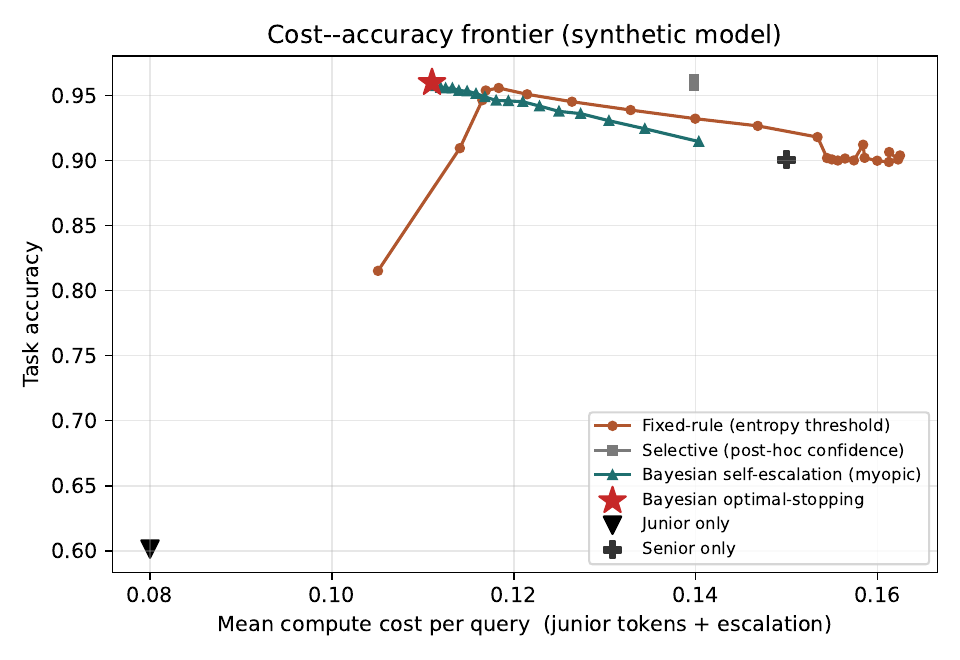}
\caption{Cost--accuracy frontier on the simulation model. The Bayesian policies
(teal) dominate the fixed-rule (orange) and post-hoc selective (grey) baselines;
the optimal-stopping point (red star) beats unconditional escalation on both axes.}
\label{fig:frontier}
\end{figure}

\paragraph{Belief dynamics and the threshold schedule.}
Figure~\ref{fig:beliefs} plots posterior trajectories: beliefs for eventual
successes drift up and for eventual failures drift down, so a threshold cleanly
separates them within a few tokens. The computed schedule illustrates
Remark~\ref{rem:nonmono}'s small-$\kappa$ regime: interior thresholds stay low
(rising gently from $0.02$ to $0.08$ across the generation)---because escalating
one step later costs only $\kappa=0.002$ while buying another observation, the
option value of waiting keeps interior escalation conservative---and the
threshold jumps to the analytic terminal value $q-\gamma/L=0.75$ exactly at the
horizon, where no further information can arrive. The myopic schedule
\eqref{eq:myopic}, by contrast, is \emph{highest} early; in this regime it
over-escalates at the start of generation.

\begin{figure}[t]
\centering
\includegraphics[width=0.66\textwidth]{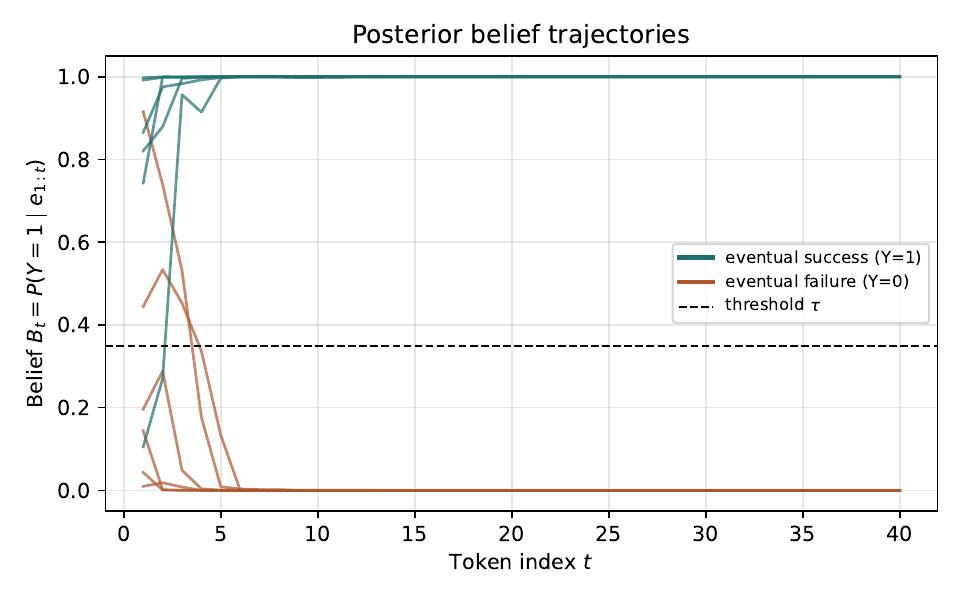}
\caption{Posterior success-belief trajectories $B_t$ for twelve queries. Eventual
successes (teal) separate from eventual failures (orange) within a few tokens; the
dashed line is an illustrative threshold.}
\label{fig:beliefs}
\end{figure}

\paragraph{Calibration sensitivity.}
To probe Proposition~\ref{prop:regret} we contaminate the stream with
``confidently wrong'' queries: a fraction of $Y=0$ cases whose signals are drawn
from the success distribution. Figure~\ref{fig:calib} shows accuracy falling from
$95.2\%$ at $0\%$ contamination to $85.7\%$ at $30\%$, while the escalation rate
\emph{drops} (from $0.47$ to $0.37$) because the contaminated cases look confident
and are wrongly kept local. This is exactly the $O(L)$ regret event of
Proposition~\ref{prop:regret} and underlines that belief calibration, not the
decision rule, is the binding constraint.

\begin{figure}[t]
\centering
\includegraphics[width=0.62\textwidth]{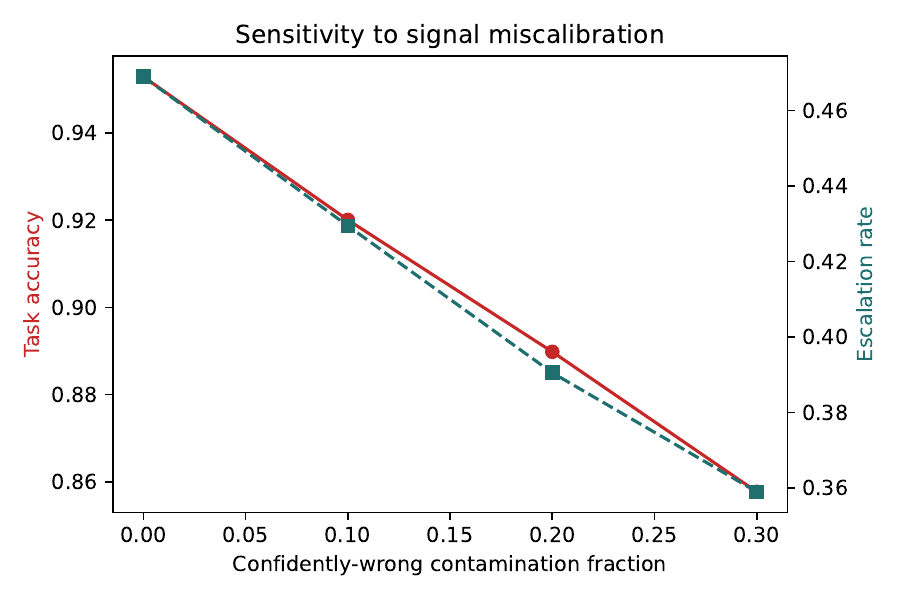}
\caption{Sensitivity to signal miscalibration. As confidently-wrong contamination
grows, accuracy degrades and---counter-productively---the escalation rate falls,
because miscalibrated confidence suppresses the very escalations that are needed.}
\label{fig:calib}
\end{figure}

\paragraph{Sample-complexity check.}
To test Theorem~\ref{thm:sample}'s driver, we discretise the signal into
$K=20$ bins, fit the plug-in $\hat\lambda$ from $n$ labelled trajectories
(add-one smoothing), and compare against the exact binned oracle on a fixed
evaluation set, averaging over $30$ calibration resamples per $n$. The belief
error $\E|\hat B_{10}-B_{10}|$ falls from $2.9\times10^{-3}$ at $n=25$ to
$2.1\times10^{-4}$ at $n=3200$, with log--log slope $-0.53$---matching the
predicted $1/\sqrt{n}$ rate. The realised end-to-end cost gap is already below
$0.0016$ (about $1\%$ of total cost) at $n=25$ and sits far under the bound at
every $n$: with a well-separated signal, beliefs rarely linger near the
threshold, so estimation errors rarely flip decisions. This is the theorem's
pessimism working as intended---the bound is a worst-case guarantee, and the
practical message is that modest labelled sets suffice when the signal is
informative.

\paragraph{Robustness across observation models.}
To check that the conclusions are not an artifact of the Beta observation
family, we repeat the comparison under two further signal models
(Table~\ref{tab:robust}): (i) \emph{Gaussian} class-conditionals with equal
variance, conditionally i.i.d.; and (ii) an \emph{AR(1)-correlated}
Gaussian model with the same marginals but lag-one correlation $\varphi=0.6$,
evaluated while the belief update \emph{still assumes independence}---a
deliberate violation of Assumption~\ref{ass:iid} that makes the update
overconfident (it double-counts correlated evidence). For each model we sweep
each policy's operating parameter and report its best total-cost point. The
Bayesian rule attains lower total cost than the tuned fixed rule in all three
models. Under misspecification its edge narrows but does not invert (total cost
$0.158$ vs.\ $0.166$), and accuracy degrades gracefully ($95.6\%$ vs.\ $95.8\%$
in the matched i.i.d.\ model): correlation costs performance, consistent with
the calibration analysis, but does not break the method.

\begin{table}[t]
\centering
\caption{Robustness across observation models. Each policy is tuned to its best
total-cost operating point per model; total cost $=$ compute $+$ expected error
cost. AR(1) uses the i.i.d.\ belief update on correlated signals
(misspecified). Numbers produced by the accompanying robustness script.}
\label{tab:robust}
\begin{tabular}{llccc}
\toprule
Observation model & Policy & Accuracy & Compute & Total cost \\
\midrule
Beta, i.i.d.\ (baseline) & Fixed-rule & $0.956$ & $0.117$ & $0.161$ \\
                         & Bayesian myopic & $0.960$ & $0.111$ & $\mathbf{0.151}$ \\
\midrule
Gaussian, i.i.d.\        & Fixed-rule & $0.958$ & $0.117$ & $0.160$ \\
                         & Bayesian myopic & $0.958$ & $0.111$ & $\mathbf{0.153}$ \\
\midrule
Gaussian AR(1), $\varphi{=}0.6$ & Fixed-rule & $0.954$ & $0.120$ & $0.166$ \\
(misspecified update)    & Bayesian myopic & $0.956$ & $0.114$ & $\mathbf{0.158}$ \\
\bottomrule
\end{tabular}
\end{table}

\section{Real-Model Validation}\label{sec:real}

The simulation of Section~\ref{sec:exp} verifies the derived policy in a world
where the modelling assumptions hold by construction. Version~1 of this paper deferred any real-model evaluation, reporting only the
pre-registered protocol of Section~\ref{sec:protocol}. This section, added in
Version~1.1, takes a first, deliberately narrow step toward that protocol and
runs the framework on a real hierarchical system. We are candid about scope: this
is a single code-generation cascade, one model pair, greedy decoding, a single
seed, and---importantly---a \emph{constant-threshold} instantiation with a
belief fit from running signal summaries rather than the recursive filter of
Lemma~\ref{lem:update}. It is evidence that the competence signal and the
escalation frontier behave as the theory predicts; it is \emph{not} a test of the
optimal-stopping dynamic program, which we leave to future work
(Section~\ref{sec:limitations}). All numbers come from reproducible scripts with
fixed seeds.

\paragraph{Setup.} The junior is \texttt{Qwen2.5-Coder-1.5B-Instruct} and the
senior \texttt{Qwen2.5-Coder-7B-Instruct}, both served locally with token-level
log-probabilities. We evaluate on the sanitized MBPP test split ($257$ tasks),
defining $Y$ by execution against the task's unit tests---an unambiguous
correctness label. For each greedy junior generation we log per-step token
entropy, next-token log-probability, and top-2 margin. The competence posterior
$B_t$ is a logistic regression on the running (cumulative) means of these three
signals, evaluated with $5$-fold cross-validation so that every reported belief
is out-of-fold. The senior attempts every task once (greedy), giving the
counterfactual needed to price escalation.

\paragraph{Capability gap and escalation ceiling.} The junior solves $62.3\%$ of
tasks and the senior $80.9\%$. Of the $97$ junior failures the senior rescues
$54$ ($55.7\%$); the remaining $43$ ($44.3\%$) are failed by both models and
form an irreducible floor no routing policy can cross. The escalation ceiling---
junior successes plus every failure escalated---is therefore
$214/257=83.3\%$, not the senior's marginal $80.9\%$. This shared-failure floor
is the real-model face of the constant-$q$ caveat in Section~\ref{sec:exp}:
senior reliability is difficulty-correlated, so the achievable gain is bounded
well below perfect rescue.

\paragraph{The competence signal is informative but imperfect.} The
cross-validated posterior attains AUROC $0.758$ against eventual success. The
per-step signal separates the classes early---mean entropy on eventual failures
exceeds that on successes from $\sim\!5\%$ of the generation---but
\emph{non-monotonically}: through the $25$--$55\%$ band the instantaneous gap
collapses and briefly inverts, precisely where confidently-wrong failures
(low-entropy, incorrect) coincide with hard-but-correct successes (high-entropy,
correct). This is the mechanism of Remark~\ref{rem:learned} and
Proposition~\ref{prop:regret} observed directly: raw entropy is not the
posterior, and confident errors are the binding failure mode.

\paragraph{Cumulative belief discrimination rises in $t$ (prediction (b)).}
Although the \emph{instantaneous} signal is non-monotonic, the \emph{cumulative}
posterior $B_t$ is not: its discrimination increases near-monotonically over the
generation (Spearman $\rho=0.93$ between generation fraction and AUROC of $B_t$;
AUROC rising from $0.51$ to $0.76$). Through the middle band where the
instantaneous signal collapses, $B_t$ \emph{plateaus} rather than declining---
confidently-wrong tokens stop contributing fresh evidence, but the belief retains
what it accumulated earlier---and resumes rising thereafter
(Figure~\ref{fig:real-auroc}). This confirms pre-registered prediction (b) of
Section~\ref{sec:protocol} and is direct support for the paper's central design
choice: accumulate a calibrated belief rather than react to per-step uncertainty.
We note the corresponding tension for early action: discrimination peaks at the
horizon ($t{=}T$), so any policy that stops early necessarily acts on a
weaker-than-terminal belief.

\paragraph{The escalation frontier dominates post-hoc routing (prediction (a)).}
We compare a \emph{streaming} policy---escalate at the first step where the
running belief falls below a threshold $\tau$, aborting the remaining junior
generation---against \emph{post-hoc routing}, which runs the junior to completion
and escalates the least-confident fraction (the confidence-cascade family the
protocol names as a baseline). Sweeping each policy's operating parameter traces
the cost--accuracy frontier of Figure~\ref{fig:real-frontier}. Streaming
dominates post-hoc across the frontier: to reach $75\%$ accuracy it uses $30.2\%$
less total compute ($14{,}841$ vs.\ $21{,}273$ generated tokens). More strikingly,
streaming reaches $75\%$ accuracy---$+12.7$ points over the junior alone---at
essentially the junior's own compute ($0.98\times$ junior-only tokens): the
tokens saved by aborting doomed generations offset the senior calls added. At
$\tau{=}0.5$ the policy escalates $37\%$ of tasks, catching them at a mean of
$29\%$ of the way through generation, for $74.7\%$ accuracy. The advantage over
post-hoc is structural: post-hoc pays every junior generation in full before it
can route, whereas streaming stops paying for a generation the moment the belief
turns against it.

\begin{table}[t]
\centering
\caption{Real-model comparison on MBPP (sanitized test, $257$ tasks;
\texttt{Qwen2.5-Coder} $1.5$B$\to$$7$B). Compute is total generated tokens,
normalised to junior-only. Post-hoc and streaming are reported at the operating
point reaching $75\%$ accuracy; the belief is cross-validated. Numbers produced
by the accompanying harvest and analysis scripts.}
\label{tab:real}
\begin{tabular}{lccc}
\toprule
Policy & Accuracy & Compute ($\times$ junior) & Esc.\ rate \\
\midrule
Junior only                       & $0.623$ & $1.00$ & $0.00$ \\
Senior only                       & $0.809$ & ---    & $1.00$ \\
Post-hoc routing (tuned)          & $0.750$ & $1.41$ & --- \\
\textbf{Streaming (tuned $\tau$)} & $0.747$ & $\mathbf{0.98}$ & $0.37$ \\
\bottomrule
\end{tabular}
\end{table}

\begin{figure}[t]
\centering
\includegraphics[width=0.68\textwidth]{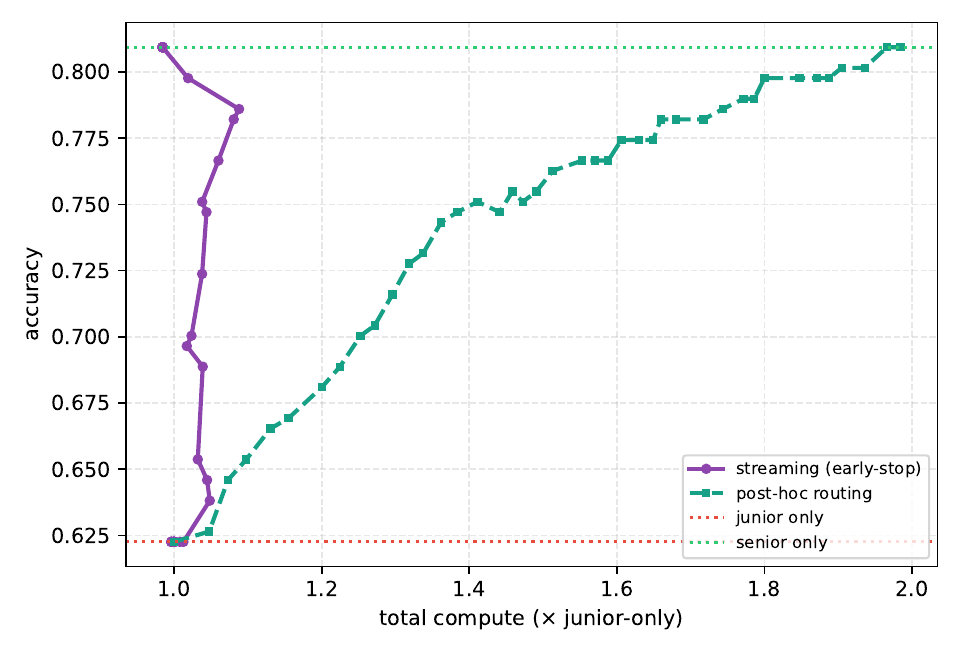}
\caption{Real-model cost--accuracy frontier. Streaming escalation (purple)
dominates post-hoc routing (teal) everywhere: it climbs from junior to senior
accuracy at near-constant compute, because aborting doomed generations offsets
the added senior calls. Dotted lines mark junior-only and senior-only accuracy.}
\label{fig:real-frontier}
\end{figure}

\begin{figure}[t]
\centering
\includegraphics[width=0.62\textwidth]{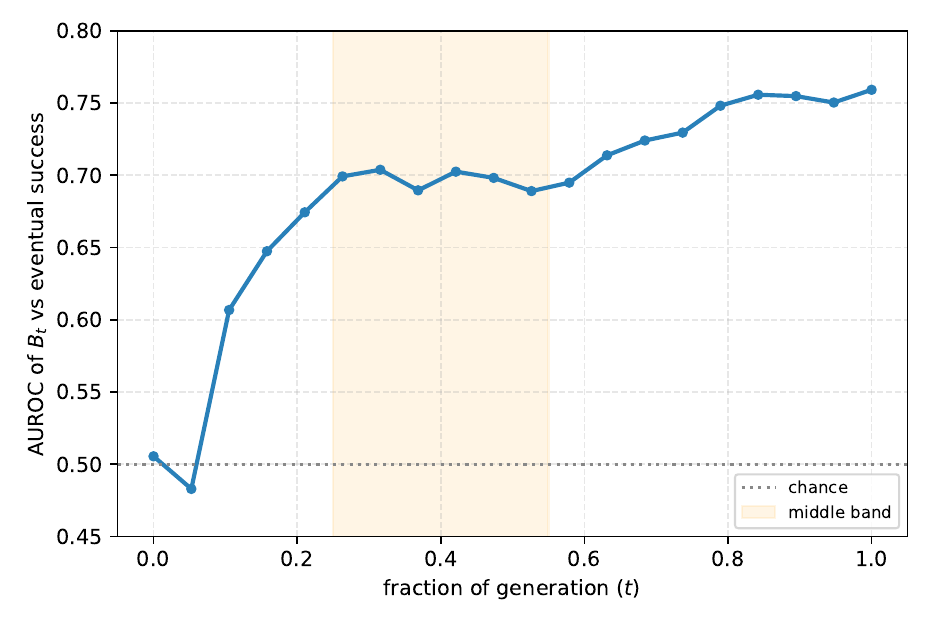}
\caption{Discrimination of the cumulative belief $B_t$ versus generation fraction
(AUROC of $B_t$ against eventual success, cross-validated). The curve rises
near-monotonically (Spearman $\rho=0.93$) and \emph{plateaus}---rather than
declining---through the shaded $25$--$55\%$ band where the instantaneous signal
collapses, confirming prediction (b).}
\label{fig:real-auroc}
\end{figure}

\paragraph{What this does and does not establish.} Two of the three pre-registered
predictions are supported on real data: the Bayesian frontier dominates the
post-hoc baseline at equal cost (a), and the cumulative belief's AUROC rises in
$t$ (b). Prediction (c)---that the accuracy gap over baselines shrinks as
calibration error grows---we do not test here; the simulation's calibration-
sensitivity study (Figure~\ref{fig:calib}) is its controlled analogue. Three
limits bound the reading. First, the policy evaluated is a \emph{constant}
threshold, not the optimal-stopping schedule of Eqs.~\eqref{eq:vT}--\eqref{eq:bellman};
by Theorem~\ref{thm:threshold} a threshold rule is the right \emph{form}, but the
option-value machinery---the contribution over classical sequential testing---is
not exercised, and the empirical margin of the dynamic program over the best
constant threshold remains to be measured. Second, the belief is a running-mean
proxy for the recursive filter of Lemma~\ref{lem:update}, as anticipated in
Section~\ref{sec:algo}. Third, this is one benchmark, one model pair, greedy
decoding, and a single seed; the protocol's reasoning and commonsense datasets,
and the sampling-based baselines, remain open. Within those limits, the framework
transfers: the signal is informative, the cumulative belief behaves as predicted,
and confidence-routed escalation is markedly more compute-efficient than routing
after the fact.

\section{Protocol for Real LLM Systems}\label{sec:protocol}

We pre-registered the full evaluation below to make the empirical test
falsifiable; it was specified in Version~1 of this paper as a plan, before any
real-model run. Section~\ref{sec:real}, added in Version~1.1, reports its first
execution on a code-generation cascade---keeping prediction and test separated in
time---and finds two of the three predictions confirmed.

\paragraph{Systems.} Junior: a small instruction/coding model served with
token-level logit access (so that entropy and next-token margin are available at
each step). Senior: a substantially stronger reasoning/coding model. Both served
through an inference stack that exposes per-token log-probabilities.

\paragraph{Signals.} Per-step $e_t$ candidates: token entropy, next-token
probability margin, and a single-pass semantic-entropy probe
\citep{kossen2024sep}. Each is calibrated separately so that comparisons isolate
the signal's quality.

\paragraph{Belief fitting and calibration.} On a labelled development split, run
the junior, record $(e_{1:T},Y)$, fit $\lambda$, and calibrate $B_t$ by isotonic
regression. Report reliability diagrams and expected calibration error (ECE) for
$B_t$ at several $t$.

\paragraph{Baselines.} (i) junior-only; (ii) senior-only; (iii) query-level router
\citep{ong2024routellm}; (iv) confidence cascade
\citep{jitkrittum2023confidence}; (v) sampling-based semantic-entropy deferral
\citep{farquhar2024semantic}; (vi) our myopic and optimal-stopping policies.

\paragraph{Datasets.} A reasoning set (e.g.\ multi-hop QA), a commonsense set with
natural easy/hard structure, and a code set with executable unit tests to define
$Y$ unambiguously; plus an in-domain deployment set.

\paragraph{Metrics.} Cost--accuracy Pareto frontier (primary); escalation
precision/recall; calibration (ECE, Brier, AUROC of $B_t$ vs.\ $Y$); and
decision/end-to-end latency, including the cost of computing $e_t$.

\paragraph{Falsifiable predictions.} If the framework transfers, then (a) the
Bayesian policy's frontier should dominate the query-level router and the
confidence cascade at equal cost; (b) AUROC of $B_t$ should rise monotonically in
$t$; and (c) the accuracy gap over baselines should shrink as calibration error
grows, per Proposition~\ref{prop:regret}. Failure of (a)--(c) would falsify the
central claims.

\section{Limitations}\label{sec:limitations}

\paragraph{Modelling assumptions.} Assumption~\ref{ass:iid} (conditional-i.i.d.\
signals) is false for real token streams, which are correlated and
non-stationary; the update then becomes an approximation and $\lambda$ should be
replaced by a sequence model of the signal. (No shape assumption on the
signal densities is needed for the threshold structure itself;
Lemma~\ref{lem:fosd} holds for any density pair.)

\paragraph{Confidently-wrong predictions.} As Proposition~\ref{prop:regret} and
Figure~\ref{fig:calib} show, the method inherits the calibration of its signal.
Where the junior is confidently wrong, no threshold policy on that signal can
help; combining epistemic signals with lightweight cross-model agreement
\citep{kolawole2024agreement} is a natural remedy.

\paragraph{Logit access.} Token-level signals require an inference stack that
exposes log-probabilities; many managed APIs return them only after generation, or
not at all, restricting deployment to self-hosted or logprob-exposing endpoints.

\paragraph{Senior independence.} We treat $q$ as constant; in practice senior
success correlates with query difficulty, and a difficulty-conditioned $q(x)$
would tighten the decision.

\paragraph{Scope of the real-model study.} Version~1 flagged the absence of
real-model results as its central limitation; Version~1.1 partially closes it.
Section~\ref{sec:real} validates the
framework on a single code-generation cascade, one model pair, greedy decoding,
and a single seed. Three limits bound it: the policy evaluated is a
\emph{constant} threshold, not the optimal-stopping schedule of
Eqs.~\eqref{eq:vT}--\eqref{eq:bellman}, so the option-value contribution over
classical sequential testing is not yet exercised empirically; the belief is a
running-mean proxy for the recursive filter of Lemma~\ref{lem:update}; and
prediction (c) of Section~\ref{sec:protocol} remains untested on real models.
The remaining datasets and baselines of the protocol are open.

\section{Conclusion}

We framed the question of \emph{when an agent should ask for help} as Bayesian
self-escalation: a junior model tracks an online posterior over its own eventual
success from the uncertainty signals it emits and defers to a stronger model when
the expected utility of deferral wins. The framework yields a closed-form myopic
threshold, an optimal-stopping characterisation with a proven threshold structure,
and a regret bound that pins the method's success to belief calibration. A
controlled simulation study confirms the predicted behaviour, including the
myopic--optimal threshold gap and the calibration-driven failure mode. The
decisive next step is the real-system protocol of Section~\ref{sec:protocol};
whether token-level signals on current LLMs are calibrated enough to realise these
gains is, in the end, an empirical question this paper is designed to make
testable.

\paragraph{Reproducibility.} The simulation study (Section~\ref{sec:exp}) is
generated by three self-contained scripts (main simulation, robustness study, and
sample-complexity experiment); the real-model validation (Section~\ref{sec:real})
by a harvest script and three post-hoc analysis scripts (separation and frontier,
streaming escalation, and the $B_t$ AUROC test). All are released with this paper.
The paper and simulation scripts are permanently archived at
\href{https://doi.org/10.5281/zenodo.21330787}{DOI: 10.5281/zenodo.21330787};
the full code, including the real-model pipeline, is at
\href{https://github.com/nadeem-shaikh/llm-self-escalation}{\texttt{github.com/nadeem-shaikh/llm-self-escalation}}.


\begin{thebibliography}{99}
\small
\bibitem[Chen et al.(2023)]{chen2023frugalgpt}
L.~Chen, M.~Zaharia, and J.~Zou.
\newblock FrugalGPT: How to use large language models while reducing cost and
improving performance.
\newblock \emph{arXiv:2305.05176}, 2023.

\bibitem[Chow(1970)]{chow1970}
C.~K.~Chow.
\newblock On optimum recognition error and reject tradeoff.
\newblock \emph{IEEE Trans.\ Information Theory}, 16(1):41--46, 1970.

\bibitem[Farquhar et al.(2024)]{farquhar2024semantic}
S.~Farquhar, J.~Kossen, L.~Kuhn, and Y.~Gal.
\newblock Detecting hallucinations in large language models using semantic
entropy.
\newblock \emph{Nature}, 630:625--630, 2024.

\bibitem[Ferguson(2006)]{ferguson2006optimal}
T.~S.~Ferguson.
\newblock \emph{Optimal Stopping and Applications}.
\newblock Electronic text, UCLA, 2006.

\bibitem[Geifman \& El-Yaniv(2017)]{geifman2017selective}
Y.~Geifman and R.~El-Yaniv.
\newblock Selective classification for deep neural networks.
\newblock In \emph{NeurIPS}, 2017.

\bibitem[Geifman \& El-Yaniv(2019)]{geifman2019selectivenet}
Y.~Geifman and R.~El-Yaniv.
\newblock SelectiveNet: A deep neural network with an integrated reject option.
\newblock In \emph{ICML}, 2019.

\bibitem[Graves(2016)]{graves2016act}
A.~Graves.
\newblock Adaptive computation time for recurrent neural networks.
\newblock \emph{arXiv:1603.08983}, 2016.

\bibitem[Jitkrittum et al.(2023)]{jitkrittum2023confidence}
W.~Jitkrittum, N.~Gupta, A.~K.~Menon, H.~Narasimhan, A.~Rawat, and S.~Kumar.
\newblock When does confidence-based cascade deferral suffice?
\newblock In \emph{NeurIPS}, 2023.

\bibitem[Kadavath et al.(2022)]{kadavath2022know}
S.~Kadavath, T.~Conerly, A.~Askell, et al.
\newblock Language models (mostly) know what they know.
\newblock \emph{arXiv:2207.05221}, 2022.

\bibitem[Kolawole et al.(2024)]{kolawole2024agreement}
S.~Kolawole, D.~Dennis, A.~Talwalkar, and V.~Smith.
\newblock Agreement-based cascading for efficient inference.
\newblock \emph{Transactions on Machine Learning Research}, 2025.

\bibitem[Kossen et al.(2024)]{kossen2024sep}
J.~Kossen, J.~Han, M.~Razzak, L.~Schut, S.~Malik, and Y.~Gal.
\newblock Semantic entropy probes: Robust and cheap hallucination detection in
LLMs.
\newblock \emph{arXiv:2406.15927}, 2024.

\bibitem[Krishnamurthy(2016)]{krishnamurthy2016pomdp}
V.~Krishnamurthy.
\newblock \emph{Partially Observed Markov Decision Processes: From Filtering to
Controlled Sensing}.
\newblock Cambridge University Press, 2016.

\bibitem[Kuhn et al.(2023)]{kuhn2023semantic}
L.~Kuhn, Y.~Gal, and S.~Farquhar.
\newblock Semantic uncertainty: Linguistic invariances for uncertainty estimation
in natural language generation.
\newblock In \emph{ICLR}, 2023.

\bibitem[Leviathan et al.(2023)]{leviathan2023speculative}
Y.~Leviathan, M.~Kalman, and Y.~Matias.
\newblock Fast inference from transformers via speculative decoding.
\newblock In \emph{ICML}, 2023.

\bibitem[Lovejoy(1987)]{lovejoy1987monotonicity}
W.~S.~Lovejoy.
\newblock Some monotonicity results for partially observed Markov decision
processes.
\newblock \emph{Operations Research}, 35(5):736--743, 1987.

\bibitem[Madaan et al.(2023)]{madaan2023selfrefine}
A.~Madaan, N.~Tandon, P.~Gupta, et al.
\newblock Self-Refine: Iterative refinement with self-feedback.
\newblock In \emph{NeurIPS}, 2023.

\bibitem[Madras et al.(2018)]{madras2018predict}
D.~Madras, T.~Pitassi, and R.~Zemel.
\newblock Predict responsibly: Improving fairness and accuracy by learning to
defer.
\newblock In \emph{NeurIPS}, 2018.

\bibitem[Mozannar \& Sontag(2020)]{mozannar2020consistent}
H.~Mozannar and D.~Sontag.
\newblock Consistent estimators for learning to defer to an expert.
\newblock In \emph{ICML}, 2020.

\bibitem[Ong et al.(2024)]{ong2024routellm}
I.~Ong, A.~Almahairi, V.~Wu, W.-L.~Chiang, T.~Wu, J.~E.~Gonzalez, M.~W.~Kadous,
and I.~Stoica.
\newblock RouteLLM: Learning to route LLMs with preference data.
\newblock \emph{arXiv:2406.18665}, 2024.

\bibitem[Peskir \& Shiryaev(2006)]{peskir2006optimal}
G.~Peskir and A.~Shiryaev.
\newblock \emph{Optimal Stopping and Free-Boundary Problems}.
\newblock Birkh\"auser, 2006.

\bibitem[Schuster et al.(2022)]{schuster2022calm}
T.~Schuster, A.~Fisch, J.~Gupta, M.~Dehghani, D.~Bahri, V.~Q.~Tran, Y.~Tay, and
D.~Metzler.
\newblock Confident adaptive language modeling.
\newblock In \emph{NeurIPS}, 2022.

\bibitem[Snell et al.(2024)]{snell2024testtime}
C.~Snell, J.~Lee, K.~Xu, and A.~Kumar.
\newblock Scaling LLM test-time compute optimally can be more effective than
scaling model parameters.
\newblock \emph{arXiv:2408.03314}, 2024.

\bibitem[Soiffer et al.(2025)]{soiffer2025semantic}
D.~Soiffer, S.~Kolawole, and V.~Smith.
\newblock Semantic agreement enables efficient open-ended LLM cascades.
\newblock In \emph{EMNLP (Industry Track)}, 2025.

\bibitem[Verma \& Nalisnick(2022)]{verma2022calibrated}
R.~Verma and E.~Nalisnick.
\newblock Calibrated learning to defer with one-vs-all classifiers.
\newblock In \emph{ICML}, 2022.
\bibitem[Wald(1945)]{wald1945}
A.~Wald.
\newblock Sequential tests of statistical hypotheses.
\newblock \emph{Annals of Mathematical Statistics}, 16(2):117--186, 1945.

\bibitem[Wald \& Wolfowitz(1948)]{waldwolfowitz1948}
A.~Wald and J.~Wolfowitz.
\newblock Optimum character of the sequential probability ratio test.
\newblock \emph{Annals of Mathematical Statistics}, 19(3):326--339, 1948.

\end{thebibliography}
\end{document}